\documentclass[runningheads]{llncs}
\usepackage[T1]{fontenc}
\usepackage{amsmath, amssymb}
\usepackage{graphicx}
\usepackage{dsfont}
\usepackage{color}
\begin{document}
\title{Multi-Target Decision Making under \\Conditions of Severe Uncertainty}
%
%
\author{Christoph Jansen\orcidID{0000-0002-5648-4687} \and
Georg Schollmeyer\orcidID{0000-0002-6199-1886}\and
Thomas Augustin\orcidID{0000-0002-1854-6226} 
}
\authorrunning{C. Jansen et al.}
%
\institute{
Department of Statistics, Ludwig-Maximilians-Universität, München, Germany\\
\email{\{firstname.lastname\}@stat.uni-muenchen.de}}
\maketitle              
\begin{abstract}
The quality of consequences in a decision making problem under (severe) uncertainty must often be compared among different targets (goals, objectives) simultaneously. In addition, the evaluations of a consequence's performance under the various targets often differ in their scale of measurement, classically being either purely ordinal or perfectly cardinal. In this paper, we transfer recent developments from abstract decision theory with incomplete preferential and probabilistic information to this multi-target setting and show how -- by exploiting the (potentially) partial cardinal and partial probabilistic information -- more informative orders for comparing decisions can be given than the Pareto order. We discuss some interesting properties of the proposed orders between decision options and show how they can be concretely computed by linear optimization. We conclude the paper by demonstrating our framework in an artificial (but quite real-world) example in the context of comparing algorithms under different performance measures.

\end{abstract}
\section{Introduction}
The basic model of decision making under uncertainty is as simple as it is expressive: an \textit{agent} is asked to choose between different available \textit{actions} $X$ from a known set of actions $\mathcal{G}$. The challenge is that the \textit{consequence} of choosing an action $X$ is not deterministic, but rather depends on which \textit{state of nature} from a known set of such states turns out to be the true one. Formally, each action is a mapping $X:S \rightarrow A$, where $A$ is the set of all possible consequences. The \textit{decision problem} $\mathcal{G}$ is then simply a subset of the set of all possible actions, i.e., the set $A^S=\{X:S \to A\}$.\footnote{For an original sources, see~\cite{s1954}.} The agent's goal here is to select an optimal action. This goal is formalized by specifying a \textit{choice function} $ch:2^{\mathcal{G}} \to 2^{\mathcal{G}}$ satisfying $ch(\mathcal{D}) \subseteq \mathcal{D} $ for all $\mathcal{D} \in 2^{\mathcal{G}}$. The sets $ch(\mathcal{D})$ are called \textit{choice sets} and have a slightly different interpretation depending on the quality of the information used to construct the choice function: The \textit{strong view} interprets $ch(\mathcal{D})$ as the set of \textit{optimal} actions from $\mathcal{D}$. The \textit{weak view}, on the other hand, interprets $ch(\mathcal{D})$ as the set of actions from $\mathcal{D}$ that \textit{cannot be rejected} based on the information.\footnote{For more details on the choice function approach to decision making see, e.g.,~\cite{b2015}.}

To construct choice functions, one mainly uses information from two different sources: The first source $I_1$ is the information about the \textit{process that generates the states of nature}. The second source $I_2$ is the information about the \textit{agent's preferences} over the consequence set $A$. In decision theory, it is classically assumed that $I_1$ has sufficient structure to be expressed in terms of a \textit{single probability measure} over the states from $S$ (see, e.g.,~\cite{f1974}), whereas $I_2$ is assumed to provide enough information to be characterized by a \textit{cardinally interpretable utility function} (see, e.g.,~\cite{krantz1971foundations}). Under these two structural assumptions, a suitable choice function is quickly found: One selects from each set $\mathcal{D}$ those actions which\textit{ maximize} the -- then well-defined -- \textit{expected utility}. Obviously, the choice sets of the choice function based on expected utility comparison can then also be given the strong interpretation.

However, in many realistic applications it turns out that the classical assumptions are systematically too restrictive and should be replaced by relaxed uncertainty assumptions and preference assumptions in order to meet the requirement of a useful theory for practice. A prominent and much discussed such application is \textit{multi-target decision problems}:\footnote{See, e.g.,~\cite{kr1993} for a classic source and~\cite{ts2017}~for recent work. It seems important to us to emphasize the difference to the (related) theory of multicriteria decision making (see, e.g.,~\cite{alv2013}~for a survey): While -- roughly speaking -- in multi-criteria decision making the same utility function is evaluated with respect to different criteria, in the multi-target setting different utility functions are evaluated under the same criterion.} By considering multiple targets simultaneously, the consequence set becomes multidimensional and generally only partially ordered, and (in general) there is no hope for a partial ordering to be be adequately described by a unique cardinal utility function. 

In this paper, we aim to contribute a new perspective to the lively discussion on multi-target decision making. To this end, we transfer recent developments from decision theory under weakly structured information -- based on both complexly structured preferences and imprecise probabilistic models -- to the multi-target situation and show how they can be used in a flexible and information-efficient way to generalize classical concepts of multi-target decision making. This transfer allows us to preserve the appeal of the classical approach while simultaneously utilizing \textit{all} available information in a perfect manner in order to pursue more informative decision theory.

Our paper is organized as follows: In Section~\ref{wsinf}, we first recall the modeling approaches for weakly structured information (Secions~\ref{wsprob} and~\ref{wspref}), then define two types of choice functions in this framework (Section~\ref{wscrit}) and, finally, give an algorithm for computing the associated choice sets (Section~\ref{computation}). In Section~\ref{adaptation}, we introduce our version of multi-target decision problems (Section~\ref{mtdp})and transfer the concepts from before to this setting (Section~\ref{transfer}). 
In Section~\ref{application} we illustrate our framework in a (synthetic) application example. Section~\ref{crem} concludes.
\section{Decision Making under Weakly Structured Information} \label{wsinf}
The decisive advantage of a generalized decision theory, which also includes a relaxation of the assumptions on $I_1$ and $I_2$, is that it is also applicable in situations in which classical decision theory fails. It is \textit{information efficient}, as it manages to include every piece of information, no matter how weakly structured. In the following two subsections we briefly present the most important basic concepts for the formal description of the relaxation of the two sources of information.
\subsection{Weakly Structured Probabilistic Information} \label{wsprob}
We now turn to the relaxation of $I_1$, i.e., the information about the process that generates the states of nature. It is classically assumed to be describable by a single probability measure. Often, however, imperfect probabilistic information will be present, rather than perfect, e.g., in the form of constraints on the probabilities of certain events or, more generally, on the expectations of certain random variables. To describe this kind of generalized uncertainty, the theory of \textit{imprecise probabilities} as developed in \cite{l1974,kofler,Walley.1991,w2001} is perfectly suitable. It should be noted here that the term imprecise probabilities is actually an umbrella term for many different generalized uncertainty theories. We restrict ourselves here to a specific one among them, namely convex finitely generated \textit{credal sets}.
\begin{definition}\label{credset}
A \textit{finitely-generated credal set} on $(S ,\sigma(S))$ is a set of the form
\begin{equation*}
\mathcal{M}=\Bigr\{ \pi \in \mathcal{P}: \underline{b}_\ell \leq \mathbb{E}_\pi(f_\ell) \leq \overline{b}_\ell \text{ for } \ell=1, \dots, r\Bigl\}
\end{equation*}
with $\mathcal{P}$ the set of all probabilities on $(S ,\sigma(S))$, $f_1, \dots , f_r: S \to \mathbb{R}$  bounded and measurable, and $\underline{b}_\ell \leq \overline{b}_\ell$ their lower and upper expectation bounds.
\end{definition}
%

It is useful that such credal sets -- at least for finite $S$ -- have only finitely many extreme points. If so, we denote  the set of these by
$\mathcal{E}(\mathcal{M})=\{\pi^{(1)}, \dots ,\pi^{(K)}\}$.    
%
\subsection{Weakly Structured Preferences} \label{wspref}
The information source $I_2$ is classically assumed to be structured enough to be described by a cardinally interpretable utility function. Relaxing this assumption, it makes sense -- comparable to the situation of relaxing the uncertainty model -- to work with the set of all utility functions which are consistent with certain preference restrictions. In order to be able to formalize even very complexly structured restrictions, we use so-called preference systems.\footnote{The following Definitions~\ref{ps},~\ref{consistency}, and~\ref{granularity} are (essentially) taken from~\cite{jsa2018} and also discussed in~\cite{jbas2021}. General representation results for the formally related concept of a \textit{difference preorder} can be found in~\cite{p2013}.}
\begin{definition}\label{ps}
Let $A$ denote a set of consequences. Let further $R_1 \subseteq A \times A$ be a pre-order\footnote{That is, reflexive and transitive.} on $A$, and
 $R_2 \subseteq R_1 \times R_1$ be a pre-order on $R_1$. The triplet $\mathcal{A}=[A, R_1 , R_2]$ is called a \textbf{preference system} on $A$. The preference system $\mathcal{A}'=[A', R_1' , R_2']$ is called \textbf{subsystem} of $\mathcal{A}$ if $A' \subseteq A$, $R_1'\subseteq R_1$, and $R_2'\subseteq R_2$.
\end{definition}
We introduce the following rationality criterion for preference systems. Here, for a pre-order $R \subseteq M \times M$, we denote by $P_R \subseteq M \times M$ its \textit{strict part}
and by $I_R \subseteq M \times M$ its \textit{indifference part}.
%
\begin{definition} \label{consistency}
The preference system $\mathcal{A}=[A, R_1 , R_2]$ is \textbf{consistent} if there exists a function $u:A \to [0,1]$ such that for all $a,b,c,d \in A$ it holds:
\begin{itemize}\itemsep1.5mm
\item[i)] If  $ (a , b) \in R_1$, then $u(a) \geq u(b)$, where equality holds iff $(a,b)\in I_{R_1}$.
\item[ii)] If $((a , b),(c,d)) \in R_2$, then $u(a) -u(b) \geq u(c)-u(d)$, where equality holds iff $((a,b),(c,d))\in I_{R_2}$.
\end{itemize}
The set of all such \textbf{representations} $u$ satisfying i) and ii) is denoted by $\mathcal{U}_{\mathcal{A}}$. 
\label{consistent}
\end{definition}
For our later decision rule, it is necessary to consider the set of all normalized representations that can be intuitively regularized using a single parameter.
\begin{definition} \label{granularity}
Let $\mathcal{A}=[A, R_1 , R_2]$ be a consistent preference system containing $a_*, a^* \in A$ such that $(a^*,a) \in R_1$ and $(a,a_*) \in R_1$ for all $a \in A$. Then 
$$\mathcal{N}_{\mathcal{A}}:= \Bigl\{u  \in \mathcal{U_{\mathcal{A}}}: u(a_*)=0 ~\wedge ~ u(a^*)=1 \Bigl\}$$
is called the \textbf{normalized representation set} of $\mathcal{A}$. Further, for a number $\delta \in [0,1)$, $\mathcal{N}^{\delta}_{\mathcal{A}}$ denotes the set of all $u \in \mathcal{N}_{\mathcal{A}}$ satisfying 
 $$ u(a)-u(b) \geq \delta ~~~\wedge ~~~ u(c)-u(d) -u(e) + u(f) \geq \delta $$ 
 for all $(a,b) \in P_{R_1}$ and all $((c,d),(e,f)) \in P_{R_2} $. Call $\mathcal{A}$ \textbf{$\delta$-consistent} if $\mathcal{N}^{\delta}_{\mathcal{A}} \neq \emptyset$.
\end{definition}
\subsection{A Criterion for Decision Making} \label{wscrit}
Naturally, a generalization of the structural assumptions to the information sources $I_1$ and $I_2$ also requires a generalization of the decision theory based on these information sources and the associated choice functions. Much work has been done in the literature on the case where the information source $I_1$ was replaced by an imprecise probabilistic model,\footnote{See, e.g.,~\cite{troffaes_ijar} for a survey or \cite{l1974,Walley.1991,festschrift} for original sources. Note that there is also quite an amount of literature on computation for that case, see, e.g.,~\cite{autkin,kikuti,compitip,jsa2017}.} while the information source $I_2$ was typically left untouched. A recent work on choice functions under generalization of both information sources simultaneously is given by~\cite{jsa2018}. We focus here on only one choice function, which is a generalization of the one induced by the relation $R_{\forall\forall}$ discussed in~\cite[p. 123]{jsa2018}. 
\begin{definition}\label{GSD}
Let $\mathcal{A}=[A, R_1 , R_2]$ be a $\delta$-consistent preference system, let $\mathcal{M}$ be a credal set on $(S,\sigma(S))$, and let
$$\mathcal{F}_{(\mathcal{A},S)}:=\Bigr\{X \in A^S: u \circ X \text{ is }\sigma(S)\text{-}\mathcal{B}_{\mathbb{R}}([0,1])\text{-measurable for all } u \in \mathcal{U}_{\mathcal{A}}\Bigl\}.$$
For $X,Y \in \mathcal{F}_{(\mathcal{A},S)}$, the variable $Y$ is called \textbf{$(\mathcal{A},\mathcal{M} , \delta)$-dominated} by $X$ if 
$$\mathbb{E}_{\pi}(u \circ X) \geq \mathbb{E}_{\pi}(u \circ Y)$$
 for all $u \in \mathcal{N}^{\delta}_{\mathcal{A}}$ and $\pi \in \mathcal{M}$. Denote the induced relation by $\geq_{(\mathcal{A},\mathcal{M} , \delta)}$.
\end{definition}
The relation $\geq_{(\mathcal{A},\mathcal{M} , \delta)}$ induces two choice functions in a perfectly natural way. The first allows the strong view and selects those actions that dominate all other actions for any compatible combination of utility and probability in expectation. The second allows only the weak view and selects those actions that are not strictly dominated by any other action.
\begin{definition}  \label{cf}
Consider the situation of Definition~\ref{GSD} and let $\mathcal{D} \subseteq\mathcal{G} \subseteq \mathcal{F}_{(\mathcal{A},S)}$. Define the following two sets associated with the relation $\geq_{(\mathcal{A},\mathcal{M} , \delta)}$:
\begin{itemize}
    \item[i)] The set of $\geq_{(\mathcal{A},\mathcal{M} , \delta)}$-\textbf{maximal acts} from $\mathcal{D} \subseteq\mathcal{G}$ is given by  
    $$\textsf{max}(\mathcal{D},\mathcal{A},\mathcal{M} , \delta):=\Bigl\{X \in \mathcal{D}:(X,Y) \in \geq_{(\mathcal{A},\mathcal{M} , \delta)} \text{ for all } Y \in \mathcal{D}  \Bigr\}.$$
    \item[ii)] The set of $\geq_{(\mathcal{A},\mathcal{M} , \delta)}$-\textbf{undominated acts} from $\mathcal{D} \subseteq\mathcal{G}$ is given by
    $$\textsf{und}(\mathcal{D},\mathcal{A},\mathcal{M} , \delta):=\Bigl\{X \in \mathcal{D}: \nexists Y\in \mathcal{D} \text{ such that }(Y,X) \in P_{\geq_{(\mathcal{A},\mathcal{M} , \delta)}}  \Bigr\}.$$
\end{itemize}
\end{definition}
The next proposition establishes a relationship between the choice sets of our two choice functions for different values of the regularization parameter $\delta$.
\begin{proposition}\label{nested}
Consider the situation of Definition~\ref{GSD} and let $\mathcal{D} \subseteq\mathcal{G} \subseteq \mathcal{F}_{(\mathcal{A},S)}$. For $0 \leq\delta_1 \leq \delta_2\leq \delta<1$ it then holds:
\begin{itemize}\itemsep1.5mm
    \item[i)] $\textsf{max}(\mathcal{D},\mathcal{A},\mathcal{M} , \delta_1)\subseteq \textsf{max}(\mathcal{D},\mathcal{A},\mathcal{M} , \delta_2)$
    \item[ii)] $\textsf{und}(\mathcal{D},\mathcal{A},\mathcal{M} , \delta_2)\subseteq \textsf{und}(\mathcal{D},\mathcal{A},\mathcal{M} , \delta_1)$
\end{itemize}
\end{proposition}
\begin{proof}
Both parts of the Proposition straightforwardly follow by observing that the condition $0 \leq\delta_1 \leq \delta_2\leq \delta<1$ together with $\delta$-consistency implies the property $\mathcal{N}^{\delta_2}_{\mathcal{A}}\subseteq\mathcal{N}^{\delta_1}_{\mathcal{A}}$ by definition. Specifically, in case i) this property implies that if $X$ dominates all $Y\in \mathcal{D}$ in expectation w.r.t.~all pairs $ (u,\pi) \in \mathcal{N}^{\delta_1}_{\mathcal{A}}\times\mathcal{M}$, then the same holds true for all pairs $ (u,\pi) \in \mathcal{N}^{\delta_2}_{\mathcal{A}}\times\mathcal{M}$. Contrarily, in case ii) the property implies that if there is no $Y \in \mathcal{D}$ which strictly expectation-dominates $X$ for some pair $ (u_0,\pi_0) \in \mathcal{N}^{\delta_2}_{\mathcal{A}}\times\mathcal{M}$ and weakly expectation-dominates $X$ for all pairs $ (u,\pi) \in \mathcal{N}^{\delta_2}_{\mathcal{A}}\times\mathcal{M}$, then it clearly does not exist such $Y$ if expectation-domination must be satisfied over the larger set of pairs $\mathcal{N}^{\delta_1}_{\mathcal{A}}\times\mathcal{M}$.\hfill $\square$
\end{proof}
Proposition~\ref{nested} nicely illustrates the role of $\delta$ in the choice functions from Definition~\ref{cf}: by coarsening the granularity at which the utility of consequences is measured, i.e., by increasing $\delta$, clearer choices can be made. Specifically, in the case of maximal actions, the choice sets increase with increasing $\delta$, which means that maximal actions can be found at all (the smaller $\delta$, the more likely the choice sets are empty). In the case of undominated actions, the choice sets decrease with increasing $\delta$ and fewer actions cannot be rejected given the information.
\subsection{Computation}\label{computation}
If the sets $A$ and $S$ are finite, we now show -- generalizing existing results as a preparation for the multi-target setting -- that checking two actions for $\geq_{(\mathcal{A},\mathcal{M} , \delta)}$-dominance can be done by solving a series of linear programs. By repeatedly applying this procedure, the choice sets of the choice functions from Definition~\ref{cf} can also be computed. An important part is that the property of being a representation of a preference system can be expressed by a set of linear inequalities.
\begin{definition}
Let $\mathcal{A}=[A, R_1 , R_2]$ be $\delta$-consistent,
 where $A=\{a_1, \dots , a_n\}$, $S=\{s_1, \dots , s_m\}$, and
    $a_{k_1}, a_{k_2} \in A$ are such that $(a_{k_1},a) \in R_1$ and $(a,a_{k_2}) \in R_1$ for all $a \in A$.
Denote by $\nabla_{\mathcal{A}}^{\delta}$ the set of all $(v_1, \dots , v_n) \in [0,1]^n$ satisfying  the following (in)equalities:
\begin{itemize}
\item[$\cdot$] $v_{k_1}=1$ and $v_{k_2}=0$,
\item[$\cdot$] $v_i=v_j $ for every pair $(a_i,a_j) \in I_{R_1}$,
    \item[$\cdot$] $v_i-v_j \geq \delta$ for every pair $(a_i,a_j) \in P_{R_1}$,
    \item[$\cdot$] $v_k-v_l= v_p -v_q $ for every pair of pairs $((a_k,a_l),(a_p,a_q)) \in I_{R_2}$
and 
\item[$\cdot$] $v_k-v_l-v_p+v_q \geq \delta$ for every pair of pairs $((a_k,a_l),(a_p,a_q)) \in P_{R_2}$.
\end{itemize}
\end{definition}
Equipped with this, we have the following Theorem regarding the computation.
\begin{theorem} \label{tlp}
Consider the same situation as described above. For $X_i,X_j \in \mathcal{G}$ and $t \in \{1 , \dots , K\}$, define the linear program
\begin{equation}\label{glp}
  \sum_{\ell=1}^{n} v_{\ell} \cdot [\pi^{(t)}(X_i^{-1}(\{a_{\ell}\}))-\pi^{(t)}(X_j^{-1}(\{a_{\ell}\}))] \longrightarrow \min_{(v_1 , \dots , v_n)\in \mathbb{R}^{n}}
\end{equation}
  with constraints  $(v_1 , \dots , v_n)\in\nabla^{\delta}_{\mathcal{A}}$. 
  Denote by $opt_{ij}(t)$ the optimal value of this programming problem. It then holds:
  $$X_i\geq_{(\mathcal{A},\mathcal{M} , \delta)} X_j ~~\Leftrightarrow ~~ \min\{opt_{ij}(t):t=1, \dots, K\} \geq 0.$$
\end{theorem}
%
%
\begin{proof}
The proof for the case $\mathcal{M}=\{\pi\}$ is a straightforward generalization of the one of Proposition 3 in~\cite{jnsa2022}. For the case of a general convex and finitely generated credal set $\mathcal{M}$ with extreme points $\mathcal{E}(\mathcal{M})=\{\pi^{(1)}, \dots ,\pi^{(K)}\}$, we first observe that the following holds:
\begin{equation}\label{help}
X_i\geq_{(\mathcal{A},\mathcal{M} , \delta)} X_j ~~\Leftrightarrow ~~ \forall t= 1 ,\dots , K:~X_i\geq_{(\mathcal{A},\{\pi^{(t)}\} , \delta)} X_j
\end{equation}
Here, the direction $\Rightarrow$ follows by definition and the direction $\Leftarrow$ is an immediate consequence of the fact that the \textit{concave} function $\pi \mapsto \inf_{u} \mathbb{E}_{\pi}(u \circ X_i)-\mathbb{E}_{\pi}(u \circ X_j)$ must attain its minimum on $\mathcal{E}(\mathcal{M})$. Since we already observed that the Theorem is true for the case $\mathcal{M}=\{\pi\}$ for arbitrary $\pi \in \mathcal{M}$, the right hand side of~(\ref{help}) is equivalent to saying $\forall t= 1 , \dots , T:~ opt_{ij}(t) \geq 0$, which itself is equivalent to saying $\min\{opt_{ij}(t):t=1, \dots, K\} \geq 0$. This completes the proof. \hfill $\square$
\end{proof}
\section{Adaptation to Multi-Target Decision Making} \label{adaptation}
We now show how the framework for decision making under weakly structured information can be applied to the situation of multi-target decision making.

\subsection{Multi-Target Decision Making} \label{mtdp}
We now turn to the following situation: We again consider a decision problem $\mathcal{G}$ with actions $X:S \to A$ mapping from a state space $S$ to a consequence space $A$. More specifically than before, however, we now assume that the agent can evaluate the different consequences $a \in A$ with different scores that reflect their compatibility with different targets. The goal is then to determine actions that provide the most balanced good performance under all targets simultaneously.
\begin{definition}
Let $\mathcal{G}$ be a decision problem with  $A$ as its consequence set. A \textbf{target evaluation} is a function
 $\phi: A \to [0,1]$. \end{definition}
For a given target evaluation $\phi$, the number $\phi(a)$ is interpreted as a measure for $a$'s performance under the underlying target (higher is better). In what follows, we allow targets to be either of \textit{cardinal} or only of \textit{ordinal} scale: While for a cardinal target evaluation we may also compare \textit{extents of target improvement},\footnote{In the sense that $\phi(a)- \phi(b) \geq \phi (c) -\phi(d)$ allows us to conclude that the improvement from exchanging $b$ by $a$ is at least as high than the one from exchanging $d$ by $c$.}~an ordinal target evaluation forbids such comparisons and is restricted to only comparing the ranks of consequences induced by it. We now give the definition of a multi-target decision problem and introduce two important associated sets.
\begin{definition}
Let $\phi_1, \dots, \phi_r$ be distinct target evaluations for $\mathcal{G}$. Then:
\begin{itemize}\itemsep1.5mm
    \item[i)] $\mathbb{M}=(\mathcal{G},(\phi_j)_{j=1 , \dots , r})$ is called \textbf{multi-target decision problem (MTDP)}.
    \item[ii)] $X \in \mathcal{G}$ is \textbf{uniformly optimal} if $\phi_j(X(s))\geq \phi_j(Y(s))$ for every $Y \in \mathcal{G}$, $s \in S$ and $j \in \{1, \dots, r\}$. Denote the set of all such $X$ by \textsf{uno}$(\mathbb{M})$.
    \item[iii)] $X \in \mathcal{G}$ is \textbf{undominanted} if there is no $Y \in \mathcal{G}$ such that for all $s \in S$ and $j\in \{1 , \dots , r\}$ it holds $\phi_{j}(Y(s))\geq \phi_{j}(X(s))$ and for some $j_0 \in \{1 , \dots , r\}$ and $s_0 \in S$ it holds $\phi_{j_0}(Y(s_0))> \phi_{j_0}(X(s_0))$. The set of all such $X$ is denoted by \textsf{par}$(\mathbb{M})$ and called the \textbf{Pareto front} of the MTDP.
\end{itemize}
\end{definition}
\subsection{Transferring the Concepts} \label{transfer}
In a MTDP as just described, the goodness of actions must be evaluated in a multidimensional space, namely the $ [0,1]^r$, with the additional restriction that not all targets may be interpreted on a cardinal scale of measurement. Instead of restricting oneself here exclusively to the consideration of the Pareto front of the component-wise order, one can use the available information more efficiently by defining beforehand a suitable preference system on $ [0,1]^r$, which can also include the information in the cardinal dimensions. 

To do so, we assume --  w.l.o.g.~-- that the first $0 \leq z \leq r$ target evaluations $\phi_1, \dots, \phi_z$ are of cardinal scale, while the remaining ones are purely ordinal. Concretely, we then consider subsystems of the consistent  preference system\footnote{A representation is given by $u: [0,1]^r \to   [0,1]$ with $u(x)= \frac{1}{r} \sum_{i=1}^{r}x_i$ for $x \in [0,1]^r$. Even if $\text{pref}( [0,1]^r)$ is not $\delta$-consistent for $\delta>0$, its subsystems might very well be.}
\begin{equation}\label{cps}
\text{pref}( [0,1]^r)=[ [0,1]^r,R_1^*,R_2^*]   
\end{equation}
where
\begin{equation*}
R_1^*=\Bigl\{(x,y) \in  [0,1]^r\times  [0,1]^r: x_j \geq y_j \text{ for all } j=1, \dots,r \Bigr\},~\text{and}
\end{equation*}
\begin{equation*}
R_2^*= \Biggl\{((x,y),(x',y'))\in R_1^{*} \times R_1^{*}: \begin{array}{lr}
        x_j -y_j \geq x_j'-y_j' \text{ ~for all } j=1, \dots,z~~~\wedge\\
        x_j \geq x_j'\geq y_j' \geq  y_j \text{ for all } j=z+1, \dots,r
    \end{array}\Biggr\}.  
\end{equation*}
While $R_1^*$ can directly be interpreted as a componentwise dominance decision, the construction of the relation $R_2^*$ deserves a few additional words of explanation: One pair of consequences is preferred to another such pair if it is ensured in the ordinal dimensions that the exchange associated with the first pair is not a deterioration to the exchange associated with the second pair and, in addition, there is component-wise dominance of the differences of the cardinal dimensions.

The introduction of a suitable preference system now allows us to transfer the relation $\geq_{(\mathcal{A},\mathcal{M} , \delta)}$ from Definition~\ref{GSD} from general decision problems to multi-target decision problems. Here, if $\mathbb{M}=(\mathcal{G},(\phi_j)_{j=1 , \dots , r})$ is a MTDP, we denote by sub$(\mathbb{M})$ the subsystem of $\text{pref}( [0,1]^r)$ obtained by restricting $R_1^*$ and $R_2^*$ to $$\phi(\mathcal{G}):=\{\phi\circ Z(s):Z \in \mathcal{G} \wedge s \in S\}\cup\{\mathbf{0},\mathbf{1}\}$$ where $\phi\circ Z:=(\phi_1 \circ Z, \dots , \phi_r\circ Z)$ for $Z \in \mathcal{G}$ and $\mathbf{0},\mathbf{1} \in [0,1]^r$ are the vectors containing only $0$ and $1$. Further, we define $\mathcal{G}^*:=\{\phi\circ Z:Z \in \mathcal{G}\}$.
\begin{definition} \label{deldom}
Let $\mathbb{M}=(\mathcal{G},(\phi_j)_{j=1 , \dots , r})$ be a MTDP such that the function $\phi\circ Z\in \mathcal{F}_{(\emph{sub}(\mathbb{M}),S)}$ for all $Z \in \mathcal{G}$ and let $\mathcal{M}$ denote a credal set on $(S,\sigma(S))$. For $X,Y \in \mathcal{G}$, say that $Y$ is \textbf{$\delta$-dominated} by $X$, if $$\phi \circ X\geq_{(\emph{sub}(\mathbb{M}),\mathcal{M},\delta)}\phi \circ Y.$$
The induced binary relation on $\mathcal{G}$ is denoted by $\succsim_{\delta}$.
\end{definition}
It is immediate that $X \in \mathcal{G}$ is maximal resp. undominated w.r.t.~$\succsim_{\delta}$ if and only if $X \in \textsf{max}(\mathcal{G}^*,\text{sub}(\mathbb{M}),\mathcal{M} , \delta)$ resp. $X \in \textsf{und}(\mathcal{G}^*,\text{sub}(\mathbb{M}),\mathcal{M} , \delta)$. Given this observation, the following Proposition demonstrates that $\succsim_{\delta}$ is (in general) more informative than a simple Pareto-analysis.
\begin{proposition}
Consider the situation of Definition~\ref{deldom}. If $S$ is (at most) countable and  $\pi(\{s\})>0$ for all $\pi \in \mathcal{M}$ and $s \in S$,\footnote{These conditions are only needed for property ii), whereas i) holds in full generality.} the  following properties hold:
\begin{itemize}\itemsep1.5mm
    \item[i)] $\textsf{max}(\mathcal{G}^*,\emph{sub}(\mathbb{M}),\mathcal{M} , \delta)\supseteq\textsf{uno}(\mathbb{M})$
    \item[ii)] $\textsf{und}(\mathcal{G}^*,\emph{sub}(\mathbb{M}),\mathcal{M} , \delta) \subseteq \textsf{par}(\mathbb{M})$
\end{itemize}
Moreover, there is equality in i) and ii) if the restriction of $R_2^*$ is empty, $\delta =0$ and $\mathcal{M}$ is the set of all probability measures.
\end{proposition}
\begin{proof}
i) If $X \in \textsf{uno}(\mathbb{M})$, then for all $Y \in \mathcal{G}$ and $s \in S$ we have component-wise dominance of $\phi \circ X(s) $ over $\phi \circ Y(s) $. Thus, $(\phi \circ X(s),\phi \circ Y(s))$ is in the restriction of $ R_1^*$ for all $s\in S$. If we choose $u \in \mathcal{N}^{\delta}_{\text{sub}(\mathbb{M})}$ and $\pi \in \mathcal{M}$ arbitrarily, this implies $\mathbb{E}_{\pi}(u \circ \phi \circ X)\geq\mathbb{E}_{\pi}(u \circ \phi \circ Y)$, since $u$
is isotone w.r.t.~the restriction of $ R_1^*$ and the expectation operator respects isotone transformations.
%
%

ii) Let $X \in \textsf{und}(\mathcal{G},\text{sub}(\mathbb{M}),\mathcal{M} , \delta)$ and assume $X \notin  \textsf{par}(\mathbb{M})$. Then, there is $Y \in \mathcal{G}$ s.t.~for all $s \in S$ and $j\in \{1 , \dots , r\}$ it holds $\phi_{j}(Y(s))\geq \phi_{j}(X(s))$ and for some for some $j_0 \in \{1 , \dots , r\}$ and $s_0 \in S$ it holds $\phi_{j_0}(Y(s_0))> \phi_{j_0}(X(s_0))$. This implies $(\phi \circ Y(s),\phi \circ X(s))$ is in the restriction of $ R_1^*$ for all $s \in S$ and $(\phi \circ X(s_0),\phi \circ Y(s_0))$ is in the restriction of $ P_{R_1^*}$. Choose $\pi \in \mathcal{M}$ and $u \in \mathcal{N}^{\delta}_{\text{sub}(\mathbb{M})}$ arbitrary and define $f:=u \circ \phi \circ Y - u \circ \phi \circ X$. Then, we can compute
\begin{eqnarray*}
\mathbb{E}_{\pi_0}(u_0 \circ \phi \circ Y)-\mathbb{E}_{\pi_0}(u_0 \circ \phi \circ X)
&=& f(s_0) \cdot \pi(\{s_0\}) + \sum_{s \in S\setminus\{s_0\}} f(s) \cdot \pi(\{s\})\\
&\geq& f(s_0) \cdot \pi(\{s_0\}) >0
\end{eqnarray*}
Here, $\geq$ follows since $f \geq 0$ and $>$ follows since $f(s_0) > 0$ and $\pi(\{s_0\})>0$.  This is a contradiction to $X \in \textsf{und}(\mathcal{G},\text{sub}(\mathbb{M}),\mathcal{M} , \delta)$. \hfill  $\square$
\end{proof}
We conclude the section with a immediate consequence of Theorem~\ref{tlp}, which allows to check for $\delta$-dominance in finite MTDPs.
\begin{corollary}
Consider the situation of Definition~\ref{deldom}. If $\phi(\mathcal{G})$ is finite, checking if $(X_i,X_j) \in \succsim_{\delta}$ can be done by the linear program (\ref{glp}) from Theorem~\ref{tlp} with $X_i$, $X_j$ replaced by $\phi \circ X_i$, $\phi \circ X_j$, $A$ replaced by $\phi(\mathcal{G})$, and $\mathcal{A}$ replaced by $\emph{sub}(\mathbb{M})$.
\label{cor}
\end{corollary}
%
%
%
%
%
\section{Example: Comparison of Algorithms} \label{application}
To illustrate the framework just discussed, we will now take a (synthetic, but potentially realistic) data example: the comparison of algorithms with respect to several targets simultaneously. We assume that six different algorithms $A_1, \dots , A_6$ are to be compared with respect to three different targets, more precisely:
\begin{itemize}\itemsep1.5mm
    \item[$\phi_1$] \textsf{Running Time:} Cardinal target evaluating a score for the algorithms running time for a specific situation on a $[0,1]$-scale (higher is better).
    \item[$\phi_2$] \textsf{Performance:} Cardinal target measuring the goodness of performance of the algorithm for a specific situation on a $[0,1]$-scale (higher is better).
    \item[$\phi_3$] \textsf{Scenario Specific Explainability:} Ordinal target giving each consequence a label of explainability in $\{0,0.1,\dots, 0.9, 1\}$ between $0 \hat{=} \text{"not"}$ to $1 \hat{=} \text{"perfect"}$.
\end{itemize}
Further, suppose that the algorithms are to be compared under five  different scenarios collected in $S=\{s_1, s_2,s_3,s_4,s_5\}$ that can potentially affect the different targets and for which we assume we can rank them according to their probability of occurring. This can be formalized by the credal set
\begin{equation}
    \mathcal{M}=\Bigl\{\pi: \pi(\{s_1\}) \geq \pi(\{s_2\}) \geq \pi(\{s_3\}) \geq \pi(\{s_4\})\geq \pi(\{s_5\})\Bigl\}
\end{equation}
whose extreme points are given by $\mathcal{E}(\mathcal{M})=\{\pi^{(1)}, \dots ,\pi^{(5)}\}$ induced by the formula  $\pi^{(k)}(\{s_j\})=\frac{1}{k} \cdot \mathds{1}_{\{s_1, \dots , s_k\}}(s_j)$,
where $j,k \in \{1 , \dots , 5\}$ (see, e.g.,~\cite{k1976}). 

In this situation, we assume the target evaluations are given as in Table~\ref{data}, where -- as already described -- the first two targets are cardinal and the third one is purely ordinal (i.e., in construction (\ref{cps}) we have $z=2$ and $r=3$). Applying the (series of) linear program(s) described in Corollary~\ref{cor} to every pair of algorithms $(A_i,A_j)$ separately then allows us to specify the full order $\succsim_{\delta}$ for this specific situation. The Hasse diagrams of the partial order for three different values of $\delta$ are visualized in Figure~\ref{hassegraphs}.\footnote{For choosing the three $\delta$-values, we first computed  the maximal value $\delta_{max}$ for which the considered preference system is still $\delta_{max}$-consistent. We  did this computation by running the linear program from~\cite[Proposition 1]{jsa2018}. Then, we picked the values $\delta_{min}=0$ and $\delta_{med}=0.5 \cdot \delta_{max}$ and $\delta_{max}$.} The choice sets of the choice functions with weak interpretation ($\textsf{par}(\cdot)$ and $\textsf{und}(\cdot)$) are given in Table~\ref{weak}, the ones of the choice functions with strong interpretation ($\textsf{uno}(\cdot)$ and $\textsf{max}(\cdot)$) are given in Table~\ref{strong}.

The results show that, among the weakly interpretable choice functions, $\textsf{par}(\cdot)$ is the least decisive one, not rejecting a single one of the available algorithms. In contrast, the choice function $\textsf{und}(\cdot)$ can already exclude half of the available algorithms for a minimum threshold of $\delta=0$. For increasing threshold values, more and more algorithms can be excluded: While for $\delta_{med}$ $A_1$ and $A_4$ are still possible, for $\delta_{max}$ only $A_1$ is potentially acceptable. Among the choice functions with strong interpretation, only $\textsf{max}(\cdot)$ under the maximum threshold $\delta_{max}$ produces a non-empty choice set: Here $A_1$ is uniquely chosen.

As a conclusion we can say that it can be worthwhile -- at least in our synthetic application -- to include available partial knowledge about probabilities and preferences in the decision process: We obtain more informative choice and rejection sets because we can perfectly exploit the available information and do not have to ignore -- as under a pure Pareto analysis, for example -- available partial knowledge about probabilities and preferences.
\begin{table}[!h]
\begin{minipage}{.3\textwidth}
 \flushleft
\centering
\begin{tabular}{|r|rrrrr|}
  \hline
$\phi_1$ & $s_1$ & $s_2$ & $s_3$ & $s_4$ & $s_5$ \\ 
  \hline
$A_1$ & 0.81 & 0.86 & 0.68 & 0.72 & 0.56  \\ 
  $A_2$ & 0.64 & 0.82 & 0.62 & 0.93 & 0.68 \\ 
  $A_3$ & 0.60 & 0.66 & 0.62 & 0.73 & 0.58 \\ 
  $A_4$ & 0.75 & 0.66 & 0.97 & 0.83 & 0.64 \\ 
  $A_5$ & 0.33 & 0.30 & 0.53 & 0.38 & 0.44 \\ 
  $A_6$ & 0.00 & 0.21 & 0.56 & 0.12 & 0.72 \\ 
   \hline
\end{tabular}
\end{minipage}
 \hfill
\begin{minipage}{.3\textwidth}
 \flushleft
 \centering
\begin{tabular}{|r|rrrrr|}
  \hline
$\phi_2$ & $s_1$ & $s_2$ & $s_3$ & $s_4$ & $s_5$ \\ 
  \hline
$A_1$ & 0.71 & 0.88 & 0.82 & 0.90 & 0.91 \\ 
  $A_2$ & 0.52 & 0.67 & 0.68 & 0.72 & 0.88 \\ 
  $A_3$ & 0.56 & 0.45 & 0.81 & 0.83 & 0.47 \\ 
  $A_4$ & 0.36 & 0.12 & 0.54 & 0.60 & 0.17 \\ 
  $A_5$ & 0.79 & 0.30 & 0.47 & 0.68 & 0.46 \\ 
  $A_6$ & 0.14 & 0.58 & 0.30 & 0.66 & 0.29 \\ 
   \hline
\end{tabular}
\end{minipage}
\hfill
\begin{minipage}{.3\textwidth}
 \flushleft
 \centering
\begin{tabular}{|r|rrrrr|}
  \hline
$\phi_3$ & $s_1$ & $s_2$ & $s_3$ & $s_4$ & $s_5$ \\ 
  \hline
$A_1$ & 0.70 & 1.00 & 0.80 & 0.60 & 0.90 \\ 
  $A_2$ & 0.50 & 0.80 & 0.60 & 0.50 & 0.80 \\ 
  $A_3$ & 0.50 & 0.40 & 0.60 & 0.40 & 0.70 \\ 
  $A_4$ & 0.70 & 0.40 & 0.70 & 0.70 & 0.30 \\ 
  $A_5$ & 0.60 & 0.20 & 0.20 & 0.30 & 0.30 \\ 
  $A_6$ & 0.10 & 0.20 & 0.40 & 0.30 & 0.20 \\ 
   \hline
\end{tabular}
\end{minipage}
\caption{Synthetic data for the three different targets.}
\label{data}
 \end{table}
\begin{figure}[!h]
\centering
\includegraphics[width=10cm]{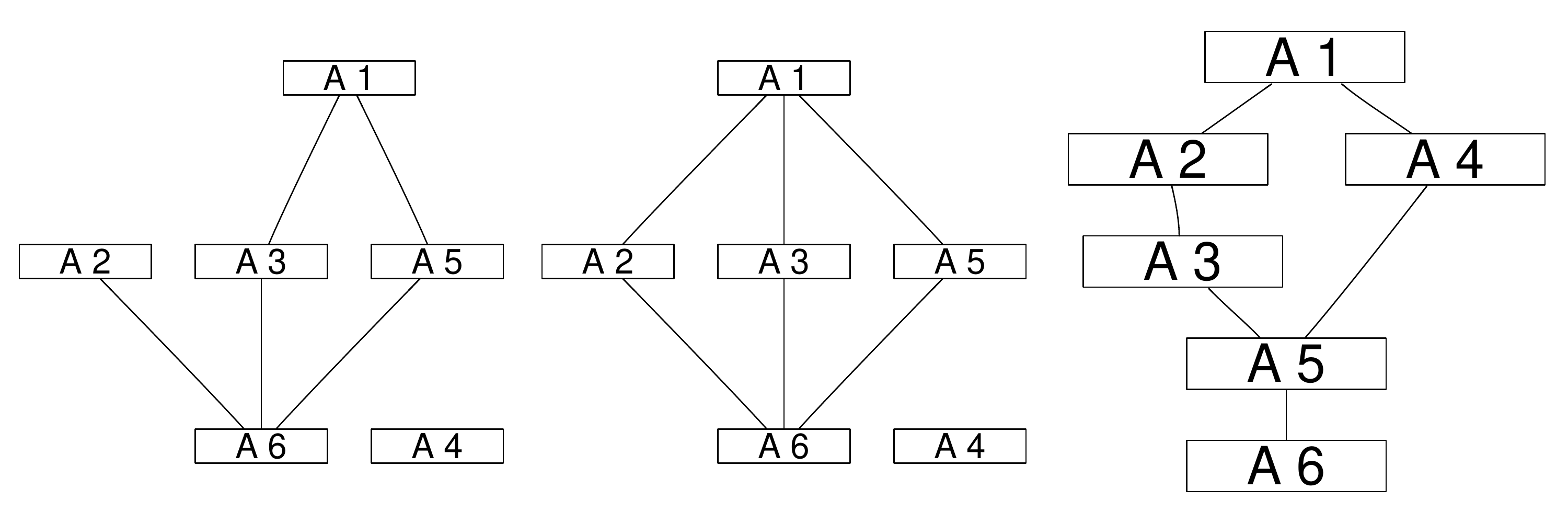}
\caption{Hasse diagrams of $\succsim_{\delta}$ for $\delta=0$ (left), $\delta_{med}$ (middle) and $\delta_{max}$ (right).}
\label{hassegraphs}
\end{figure}
\begin{table}[!h]
\centering
\begin{tabular}{|c|c|c|c|}
  \hline
  $\textsf{par}(\mathbb{M})$& $\textsf{und}(\mathcal{G}^*,\emph{sub}(\mathbb{M}),\mathcal{M} , 0)$&$\textsf{und}(\mathcal{G}^*,\emph{sub}(\mathbb{M}),\mathcal{M} , \delta_{med})$&$\textsf{und}(\mathcal{G}^*,\emph{sub}(\mathbb{M}),\mathcal{M} , \delta_{max})$\\
  \hline
  $\{A_1 , \dots, A_6\}$&$\{A_1, A_2,A_4\}$ &$\{A_1,A_4\}$&$\{A_1\}$\\
  \hline
\end{tabular}
\caption{Choice sets of the choice functions with weak interpretation.}
\label{weak}
\end{table}
\begin{table}[!h]
\centering
\begin{tabular}{|c|c|c|c|}
  \hline
  $\textsf{uno}(\mathbb{M})$& $\textsf{max}(\mathcal{G}^*,\emph{sub}(\mathbb{M}),\mathcal{M} , 0)$&$\textsf{max}(\mathcal{G}^*,\emph{sub}(\mathbb{M}),\mathcal{M} , \delta_{med})$&$\textsf{max}(\mathcal{G}^*,\emph{sub}(\mathbb{M}),\mathcal{M} , \delta_{max})$\\
  \hline
  $\emptyset$&$\emptyset$ &$\emptyset$&$\{A_1\}$\\
  \hline
\end{tabular}
\caption{Choice sets of the choice functions with strong interpretation.}
\label{strong}
\end{table}
\section{Concluding Remarks} \label{crem}
In this paper, we have further developed recent insights from decision theory under weakly structured information and transferred them to multi-target decision making problems. It has been shown that within this formal framework all the available information can be exploited in the best possible way and thus -- compared to a classical Pareto analysis -- a much more informative decision theory can be pursued. Since this initially theoretical finding has also been confirmed in our synthetic data example, a next natural step for further research is applications of our approach to real data situations. Since for larger applications also very large linear programs arise when checking the proposed dominance criterion, it should also be explored to what extent the constraint sets of the linear programs can still be purged of redundancies (e.g., by explicitly exploiting transitivity) or to what extent the optimal values can be approximated by less complex linear programs.
\subsubsection{Acknowledgements} 
Georg Schollmeyer gratefully acknowledges the financial and general support of the LMU Mentoring Program.
%
%
%

\end{document}